\documentclass[letterpaper]{article}
\usepackage{proceed2e}
\usepackage[margin=1in]{geometry}

\usepackage{amsmath,amsbsy,amsfonts,amssymb,amsthm}
\usepackage{algorithm,algorithmic,mathtools}
\usepackage{color,cases}
\usepackage{tikz}
\usetikzlibrary{calc,shapes}
\usepackage{natbib,bbm,boldline,pifont}
\usepackage{hyperref}
\usepackage{comment,enumitem}
\usepackage{xcolor,url,verbatim}

\newcommand{\cmark}{\ding{51}}
\newcommand{\xmark}{\ding{55}}

\DeclareMathOperator{\Supp}{Supp}

\DeclareMathOperator{\rank}{rank}

\DeclareMathOperator{\sym}{Sym}

\newcommand{\abs}[1]{\left\lvert #1 \right\rvert}
\newcommand{\paran}[1]{\left( #1 \right)}

\newcommand{\zeronorm}[1]{\left\lVert #1 \right\rVert_{0}}
\newcommand{\zinorm}[1]{\left\lVert #1 \right\rVert_{0,\infty}}
\newcommand{\iznorm}[1]{\left\lVert #1 \right\rVert_{\infty,0}}
\newcommand{\onenorm}[1]{\left\lVert #1 \right\rVert_{1}}
\newcommand{\twonorm}[1]{\left\lVert #1 \right\rVert_{2}}
\newcommand{\trnorm}[1]{\left\lVert #1 \right\rVert_{*}}
\newcommand{\frobnorm}[1]{\left\lVert #1 \right\rVert_{F}}
\newcommand{\infnorm}[1]{\left\lVert #1 \right\rVert_{\infty}}
\newcommand{\nn}{\nonumber}

\newtheorem{claim}{Claim}[section]
\newtheorem{lemma}{Lemma}[section]
\newtheorem{theorem}{Theorem}[section]
\newtheorem{remark}{Remark}[section]
\newtheorem{example}{Example}[section]

\newcommand{\R}{\mathbb{R}}

\renewcommand{\P}{\mathcal{P}}

\usepackage{times}

\title{Provable Inductive Robust PCA via Iterative Hard Thresholding}

\author{ {\bf U.N. Niranjan\thanks{Part of work done while at the University of California Irvine and at Xerox Research Centre India.}} \\
Microsoft Corporation \\
niranjan.uma$@$microsoft.com \\
\And
{\bf Arun Rajkumar}\thanks{Part of work done while at Xerox Research Centre India.}  \\
Conduent Labs India \\
arun.rajkumar$@$conduent.com \\
\And
{\bf Theja Tulabandhula}\thanks{Part of work done while at Xerox Research Centre India.}   \\
University of Illinois Chicago \\
tt$@$theja.org
}

\begin{document}

\maketitle

\begin{abstract}
The robust PCA problem, wherein, given an input data matrix that is the superposition of a low-rank matrix and a sparse matrix, we aim to separate out the low-rank and sparse components, is a well-studied problem in machine learning. One natural question that arises is that, as in the inductive setting, if features are provided as input as well, can we hope to do better? Answering this in the affirmative, the main goal of this paper is to study the robust PCA problem while incorporating feature information. In contrast to previous works in which recovery guarantees are based on the convex relaxation of the problem, we propose a simple iterative algorithm based on hard-thresholding of appropriate residuals. Under weaker assumptions than previous works, we prove the global convergence of our iterative procedure; moreover, it admits a much faster convergence rate and lesser computational complexity per iteration. In practice, through systematic synthetic and real data simulations, we confirm our theoretical findings regarding improvements obtained by using feature information.
\end{abstract}

\section{INTRODUCTION}
\textit{Principal Component Analysis (PCA)}~\citep{pearson1901liii} is a very fundamental and ubiquitous technique for unsupervised learning and dimensionality reduction; basically, this involves finding the best low-rank approximation to the given data matrix. To be precise, one common formulation of PCA is the following:
\begin{equation}
\label{eqn:pca}
\widehat{L} = \arg \min_{L} \frobnorm{M-L} \quad \text{ s.t. } \rank(L) \leq r
\end{equation}
where $M \in \R^{n_1 \times n_2}$ is the input data matrix, where $\frobnorm{.}$ denotes the Frobenius norm of a matrix and $1 \leq r \leq \min(n_1, n_2)$. It is well-known that the constrained optimization problem given by Equation~\eqref{eqn:pca} can be solved via the Singular Value Decomposition (SVD) and truncating the resultant decomposition to the top-$r$ singular values and singular vectors yields the optimal solution~\citep{eckart1936approximation}. While this machine learning technique has umpteen number of applications, one of its main shortcomings is that it is not robust to the presence of gross outliers since the optimization involves just an $\ell_2$ objective. To address this issue, the \textit{robust PCA} technique -- given $M$ such that $M = L^*+S^*$, our aim is to find $L^*$ and $S^*$ which are low-rank and sparse matrix components respectively -- was developed. Precisely, one hopes to solve the following problem (or its equivalent formulations):
\begin{align}
\label{eqn:rpca}
\{ \widehat{L}, \widehat{S} \} = \arg \min_{L,S} & \frobnorm{M-L-S} \nn \\
& \text{ s.t. } \rank(L) \leq r, \quad \zeronorm{S} \leq z_0
\end{align}
where $\zeronorm{.}$ denotes the number of non-zero entries in a matrix, $0 \leq r \leq \min(n_1,n_2)$ and $0 \leq z_0 \leq n_1 n_2$. While Equation~\eqref{eqn:rpca} may not be always well-posed, under certain identifiability conditions, many recent works over the past decade have advanced our understanding of this problem; we briefly recap some of the existing relevant results in Section~\ref{sec:rel}.

\begin{table*}[t]
\centering
\caption{Comparison of this work to previous robust PCA works. For simplicity and brevity, we let $n_1 = n_2 = n$ and $d_1 = d_2 = d$; let the number of non-zeros per row/column of $S^*$ be $z$ and the number of non-zero entries in $S^*$ be $m$; we use $\widetilde{O}$ to suppress $\log$ factors. Note that we consider the practically important regime of $d \ll n$.}
\label{tab:summary}
\resizebox{\textwidth}{!}{
\begin{tabular}{ccccccc}
\hlineB{2}
\textbf{Work} & Features & Approach & Incoherence & Sparsity & Corruption & Comp. complexity \\
\hlineB{2}
\citep{candes2011robust} & \xmark & Convex & Strong & Random & $m = O(n^2), r = \widetilde{O}(n)$ & $O(\frac{n^3}{\sqrt{\epsilon}})$ \\
\citep{hsu2011robust} & \xmark & Convex & Weak & Deterministic & $z = O(\frac{n}{r})$ & $O(\frac{n^3}{\sqrt{\epsilon}})$ \\
\citep{netrapalli2014non} & \xmark & Non-convex & Weak & Deterministic & $z = O(\frac{n}{r})$ & $O(r^2 n^2 \log(\frac{1}{\epsilon}))$ \\
\citep{yi2016fast} & \xmark & Non-convex & Weak & Deterministic & $z = O(\frac{n}{r^{1.5}})$ & $O(r n^2 \log(\frac{1}{\epsilon}) )$ \\
\citep{chiang2016robust} & \cmark & Convex & Strong & Random & $m = O(n^2), r = \widetilde{O}(\frac{n^2}{d})$ & $O(\frac{d n^2 + d^3}{\sqrt{\epsilon}})$ \\
\hline
\textbf{This work} & \cmark & Non-convex & Weak & Deterministic & $z = O(\frac{n}{d})$ & $O((d n^2 + d^2 r) \log(\frac{1}{\epsilon}))$ \\
\hlineB{2}
\end{tabular}
}
\end{table*}

\subsection{ROBUST INDUCTIVE LEARNING: MOTIVATION}
A key point to be noted is that Equation~\eqref{eqn:rpca} does not incorporate feature information; this is the so-called \textit{transductive} setting. In practical applications, we often have feature information available in the form of feature matrices $F_1$ and $F_2$. In the low-rank matrix recovery literature, this is often incorporated as a bilinear form, $L^* = F_1^\top W^* F_2$, which models the feature interactions via the latent space characterized by matrix $W^* \in \R^{d_1 \times d_2}$; this is the so-called \textit{inductive} setting. We now present a motivating real-life situation.
\begin{example}[\textbf{Using features for collaborative filtering with grossly corrupted observations}]
\label{eg:reco_sys}
In recommendation systems, it is often the case that we have user-product ratings matrix along with side information in the form of features corresponding to each user and product. It is common in large-scale machine learning applications that the number of products and users is very large compared to the features available for each user or product. Though a user might not have used a product, we would like to infer how the user might rate that product given the user and product features -- unlike the transductive setting this is possible in, and is a key application of, the inductive learning setting. Moreover, the ratings matrix is subject to various kinds of noise including erasures and outliers -- in this work, we consider a general noise model using which robust recovery of ratings is possible.
\end{example}
It is the goal of this paper to focus on the practically useful regime of $\max(d_1,d_2) \ll \min(n_1,n_2)$.

\subsection{RELATED WORK}
\label{sec:rel}
We now present the related work in both transductive and inductive settings.
\vspace*{-7pt}
\paragraph{Transductive setting: }This is the relatively more well-explored setting. There are two main solution approaches that have been considered in the literature namely, the convex and the non-convex methods.

Convex methods entail understanding the properties of the convex relaxation of Equation~\eqref{eqn:rpca} given by:
\begin{align}
\label{eqn:cvx_rpca}
\{ \widehat{L}, \widehat{S} \} = \arg \min_{L,S} & \trnorm{L} + \lambda \onenorm{S} \nn \\
& \text{s.t. } M = L+S
\end{align}
The works of \citep{chandrasekaran2011rank} and \citep{hsu2011robust} characterize the recovery properties of the convex program assuming a weak deterministic assumption on the support of the sparse matrix that the fraction of corrupted entries; the tightest bounds are that this fraction scales as $O(1/r)$. Under a stronger model of the sparse matrix namely, uniformly sampled support, \citep{candes2011robust} show that it is possible to have $r = O(n / \log(n))$ when $z_0 = O(n^2)$ for exact recovery with high probability. Numerically, the convex program in Equation~\eqref{eqn:cvx_rpca} is most commonly solved by variants of sub-gradient descent (involving iterative soft-thresholding); the convergence rate known for trace-norm programs is $O(1 / \sqrt{\epsilon})$~\citep{ji2009accelerated} for an $\epsilon$-close solution.

The underlying theme in non-convex methods 
involves retaining the formulation in Equation~\eqref{eqn:rpca}, starting with a suitable initialization and performing alternating projections onto non-convex sets (involving iterative hard-thresholding) until convergence. The work of \citep{netrapalli2014non} provides recovery guarantees under the weaker deterministic support assumptions matching the conditions of \citep{hsu2011robust}. However, the computational complexity of their algorithm scales with rank quadratically -- to improve this, \citep{yi2016fast} propose a (non-convex) projected gradient approach while paying a cost in the permissible number of sparse corruptions, i.e., $O(1/r^{1.5})$ as opposed to $O(1/r)$. A consequence of the analysis of these non-convex methods is that they admit a faster convergence rate -- specifically, $O(\log(1 / \epsilon))$ iterations for an $\epsilon$-close solution -- as opposed to convex methods.

It is noteworthy that the matrix completion problem (see, for instance, \citep{recht2011simpler} and \citep{jain2014fast}), where the goal is to recover an incomplete low-rank matrix, is a special case of the robust PCA problem where $S^*$ is taken to be $-L^*$ for the non-observed entries. Finally, we note that the robust PCA problem has been invoked in several applications including topic modeling~\citep{min2010decomposing}, object detection~\citep{li2004statistical} and so on.

\paragraph{Inductive setting: }To the best of our knowledge, currently, there is only one other work due to \citep{chiang2016robust} which considers the robust PCA problem in the inductive setting and presents a guaranteed convex optimization procedure for solving it; incorporating additional feature information into the robust PCA problem, they solve the following convex program, known as \textit{PCPF}:
\begin{align}
\label{eqn:cvx_irpca}
\{ \widehat{W}, \widehat{S} \} = \arg \min_{W,S} & \trnorm{W} + \lambda \onenorm{S} \nn \\
& \text{s.t. } M = F_1^\top W F_2 + S
\end{align}
For this paragraph, let $m := \zeronorm{S^*}$, $W^* = U_{W^*} \Sigma_{W^*} V_{W^*}^\top$ be the SVD of $W^*$, $F_1 F_1^\top = I$, $F_2 F_2^\top = I$ and $e_i$ denote the $i^{th}$ standard basis vector in $\R^n$; the key recovery guarantee states that $r = O(n^2 / d \log(n) \log(d))$ and $m = O(n^2)$; most notably, these guarantees are derived under stronger assumptions namely, (1) strong incoherence property, i.e., $\infnorm{U_{W^*} V_{W^*}^\top} \leq \mu \sqrt{r/n_1 n_2}$, $\max_j \twonorm{U_{W^*}^\top F_1 e_j} \leq \mu_{0} \sqrt{r/n_1}$, $\max_j \twonorm{V_{W^*}^\top F_2 e_j} \leq \mu_{0} \sqrt{r/n_2}$, $\max_j \twonorm{F_1 e_j} \leq \mu_{F_1} \sqrt{d/n_1}$, $\max_j \twonorm{F_2 e_j} \leq \mu_{F_2} \sqrt{d/n_2}$ (2) random sparsity, i.e., the support of $S^*$ is drawn uniformly at random from all subsets of $[n_1] \times [n_2]$ of size $m$. Note that assumptions such as uniform support sampling may not be realistic in practice. In contrast, as we explain in Sections \ref{sec:our_cont} and \ref{sec:assume}, our work relaxes the assumptions they require while admitting a simpler algorithm, novel analysis approach and faster convergence result.

In this context, it is also to be mentioned that for the related problem of inductive matrix completion is relatively better understood; recovery guarantees are known for both the convex (see, for instance, \citep{xu2013speedup} and \citep{chiang2015matrix}) and the non-convex (e.g., \citep{jain2013provable}) approaches. Other related works based on probabilistic modeling include \citep{zhou2012kernelized} and \citep{porteous2010bayesian}.

To summarize, we position this paper with respect to other works in Table~\ref{tab:summary}. While we have highlighted the most relevant existing results, note that the list provided here is by no means comprehensive -- such a list is beyond the scope of this work.

\subsection{OUR CONTRIBUTIONS}
\label{sec:our_cont}
To the best of our knowledge, our work is the first to derive a provable and efficient non-convex method for robust PCA in the inductive setting. Our novelty and technical contributions can be summarized along the following axes:
\begin{enumerate}[nolistsep,noitemsep]
\item \textit{Assumptions (Section~\ref{sec:assume}): }We use the weakest assumptions, i.e., (1) weak incoherence conditions on only the feature matrices and (2) (weak) deterministic support of the sparse matrix.
\item \textit{Algorithm (Section~\ref{sec:algo}): }Our algorithm (IRPCA-IHT) performs simple steps involving spectral and entry-wise hard-thresholding operations.
\item \textit{Guarantees (Sections \ref{sec:sym_noiseless}, \ref{sec:sym_noisy} and \ref{sec:asymm}: }We show $\epsilon$-close recovery in both the noiseless and noisy cases for problems of general size, feature dimension, rank and sparsity; moreover, our method has the fast (linear) convergence property.
\item \textit{Experiments (Section~\ref{sec:expt}): }We substantiate our theoretical results by demonstrating gains on both synthetic and real-world experiments.
\end{enumerate}


\section{PROBLEM SETUP}

\subsection{NOTATION AND PRELIMINARIES}
Let $M = L^*+S^*$, i.e., $\{ M, L^*, S^* \} \in \R^{n_1 \times n_2}$ are matrices such that the input data matrix $M$ is the superposition of two component matrix signals namely, the low-rank component $L^*$ and the sparse component $S^*$. 
Here, $S^*$ is a sparse perturbation matrix with unknown (deterministic) support and arbitrary magnitude. In our inductive setting, side information or features are present in the bilinear form specified $L^* = F_1^\top W^* F_2$. The feature matrices are denoted as $F_1 \in \R^{d_1 \times n_1}$ and $F_2 \in \R^{d_2 \times n_2}$. Note that the feature dimensions are $d_1$ and $d_2$ such that $\max(d_1, d_2) \ll \min(n_1, n_2)$ and $W^* \in \R^{d_1 \times d_2}$ is the rank-$r$ latent matrix to be estimated where $r \leq \min(d_1,d_2)$; intuitively, this latent matrix parameter describes the interaction and correlation among the feature vectors. Now, our optimization problem is given by:
\begin{align}
\{ \widehat{W}, & \widehat{S} \} = \arg \min_{W,S} \frobnorm{M -  F_1^\top W F_2 - S} \nn \\
& \text{s.t. } \rank(W) \leq r, \zinorm{S} \leq z_2, \iznorm{S} \leq z_1
\label{eqn:irpca}
\end{align}
Here, for a matrix $A \in \R^{n_1 \times n_2}$, we define the relevant functions, $\zinorm{A} := \max_j \sum_{i=1}^{n_2} \mathbf{1}(A_{ij} \neq 0)$, $\iznorm{A} := \max_i \sum_{j=1}^{n_1} \mathbf{1}(A_{ij} \neq 0)$, $\infnorm{A} := \max_{ij} \abs{A_{ij}}$, Frobenius norm $\frobnorm{A} := \sqrt{\sum_{i=1}^{n_1} \sum_{j=1}^{n_2} A_{ij}^2}$, spectral norm $\twonorm{A} = \max_{\twonorm{x}=1,\twonorm{y}=1} x^\top A y$ for unit vectors $x \in \R^{n_1}$ and $y \in \R^{n_2}$. Next, for a matrix $A$, we denote its maximum and minimum singular value by $\sigma_{\max}(A)$ and $\sigma_{\min}(A)$ respectively, and further the condition number of $A$ is denoted by $\kappa(A) := \sigma_{\max}(A) / \sigma_{\min}(A)$. The pseudoinverse of a matrix $A$ is denoted by $B = A^\dagger$ and is computed as $B := (A^\top A)^{-1} A^\top$ where $A$ is assumed to be of full rank. Let $I$ denote the identity matrix whose size will be clear from the context. Finally, we use $e_i$ to denote the $i^{th}$ standard basis vector in the appropriate dimension, which will also be clear from the context.

\begin{remark}[\textbf{Noisy case: motivation and setup}]
\label{rem:noisy}
Note that, so far, for simplicity and clarity, we have been focusing on the case when $M = L^* + S^*$ where $L^* = F_1^\top W^* F_2$. This model posits that $W^*$ is exactly a rank-$r$ matrix and $S^*$ is exactly a sparse matrix which might not be the case in practice. Our approach, for solving Equation~\ref{eqn:irpca}, in terms of both the algorithm and the analysis, also handles the noisy case $M = F_1^\top W^* F_2 + S^* + N^*$ wherein $N^*$ is some generic bounded additive noise that renders $L^*$ approximately low-rank or $S^*$ approximately sparse.
\end{remark}

\subsection{ASSUMPTIONS}
\label{sec:assume}
We now state and explain the intuition behind the (by now standard) identifiability assumptions on the quantities involved in our optimization problem so that it is well-posed. Also, we re-emphasize specifically that Assumptions \ref{asm:incoh} and \ref{asm:sps} are much weaker and generic than previous works such as \citep{chiang2016robust}.
\begin{enumerate}[nolistsep,noitemsep]
\item \label{asm:feas} \textit{Feasibility condition: }We assume that $\text{row}(L^*) \subseteq \text{row}(F_2)$ and $\text{col}(L^*) \subseteq \text{col}(F_1^\top)$.
\item \label{asm:incoh} \textit{Weak incoherence of the feature matrices: }Let $F_1 = U_{F_1} \Sigma_{F_1} V_{F_1}^\top$ be the SVD of the feature matrix $F_1$ such that $U_{F_1} \in \R^{d_1 \times d_1}$, $V_{F_1} \in \R^{d_1 \times n_1}$ are the matrices of left and right singular vectors respectively, and $\Sigma_{F_1} \in \R^{d_1 \times d_1}$ is the diagonal matrix of singular values. Then, we assume $\max_{i} \twonorm{e_i^\top V_{F_1}} \leq \mu_{F_1} \sqrt{d_1/n_1}$ where $\mu_{F_1}$ is called the incoherence constant of matrix $F_1$. Similarly, we assume incoherence of $F_2$ as well.
\item \label{asm:sps} \textit{Bounded deterministic sparsity: }Let the number of non-zeros per row of the sparse matrix $S$ satisfy $z_1 \leq n_1 / 20 \mu^2 d_1 \kappa$; similarly, let the number of non-zeros per column of the sparse matrix $S$ satisfy $z_2 \leq n_2 / 20 \mu^2 d_2 \kappa$. Here, $\mu = \max(\mu_{F_1},\mu_{F_2})$ and $\kappa = \max(\kappa(F_1),\kappa(F_2))$.
\item \label{asm:bdd} \textit{Bounded latent matrix: }Without loss of generality, we assume that the latent matrix is bounded, ie, $\twonorm{W^*} \leq c_W$ for a global constant $c_W$.
\end{enumerate}
Having side information always need not help; otherwise, we may always generate random features and obtain improvement over transductive learning. In disallowing this, Assumption~\ref{asm:feas} is a necessary condition, which ensures that we have informative features $F_1$ and $F_2$ in the sense that they are correlated meaningfully in the latent space given by $W^*$.

In order to make the low-rank component not too sparse and distinguishable from the sparse perturbation, we make the weak incoherence assumption on the feature matrices which says that the energy of the right singular vectors of the matrices is well-spread with respect to all the co-ordinate axes. This is precisely quantified by Assumption~\ref{asm:incoh}.

\begin{algorithm}[t]
\caption{IRPCA-IHT: Inductive Robust PCA via Iterative Hard Thresholding}
\label{alg:incrpca}
\begin{algorithmic}[1]
\STATE \textbf{Input}: Grossly corrupted data matrix $M \in \R^{n_1 \times n_2}$, feature matrices $F_1 \in \R^{d_1 \times n_1}, F_2 \in \R^{d_2 \times n_2}$, true rank $r$, noise parameter $\nu$, global constant $c_W$.
\STATE \textbf{Output}: Estimated latent matrix $\widehat{W} \in \R^{d_1 \times d_2}$ and sparse perturbation matrix $\widehat{S} \in \R^{n_1 \times n_2}$.
\STATE Initialize $L_0 \leftarrow 0$ and $\zeta_0 \leftarrow 5 \mu_{F_1} \mu_{F_2} \sigma_{\max}(F_1) \sigma_{\max}(F_2) \sqrt{\frac{d_1 d_2}{n_1 n_2}} c_W + \nu$ where $\mu_{F_1}$ and $\mu_{F_2}$ are the incoherence constants as computed in Assumption~\ref{asm:incoh}.
\FOR{$t = 1, \ldots, T$}
\STATE $\zeta_t \leftarrow \frac{\mu_{F_1} \mu_{F_2} \sigma_{\max}(F_1) \sigma_{\max}(F_2) \sqrt{d_1 d_2} c_W}{5^{t-1} \sqrt{n_1 n_2}} + \nu$.
\STATE $S_t \leftarrow \P_{\zeta_t} (M - L_{t-1})$.
\STATE $W_t \leftarrow \P_r \paran{ {(F_1^\top)}^\dagger (M-S_t) (F_2)^\dagger }$.
\STATE $L_t \leftarrow F_1^\top W_{t}F_2$.
\ENDFOR
\STATE Set $\widehat{W} \leftarrow W_T$ and $\widehat{S} \leftarrow S_T$.
\RETURN $\widehat{W}, \widehat{S}$.
\end{algorithmic}
\end{algorithm}

In our problem setup we assume that a generic (possibly adversarial) deterministic sparse perturbation is added to the low-rank matrix. This is quantified by Assumption~\ref{asm:sps}. In particular, we \textit{do not} have any specific distributional assumptions on the support of the sparse matrix, and the magnitudes and signs of its non-zero entries.
\begin{remark}[\textbf{Noisy case: assumptions}]
\label{rem:noisy_asm}
To obtain recovery guarantees for the noisy case described in Remark~\ref{rem:noisy}, the only assumption on $N^*$ we have is that it is suitably well-behaved -- this is quantified by assuming $\infnorm{N^*} \leq 1 / 40 \mu^2 d \kappa^2$.
\end{remark}

\subsection{CORRUPTION RATE}
\label{sec:samp_comp}
In this work, as give in Table~\ref{tab:summary}, we refer to the rank-sparsity trade-off in Assumption~\ref{asm:sps} as `corruption rate' -- this is the allowable extent to which the model is robust to gross outliers while retaining identifiability, ie, the number of non-zeros in the sparse corruption matrix. Note that, by using features, we are always able to tolerate $\Omega(n_1 / d_1)$ (resp. $\Omega(n_2 / d_2)$) gross corruptions per row (resp. column). This is a gain over the transductive setting as in \citep{netrapalli2014non} where the permissible number of outliers is $O(n_1 / r)$ (resp. $O(n_2 / r)$) per row (resp. column) and $r$ could be potentially $O(n)$.

\subsection{ALGORITHM}
\label{sec:algo}
Our method, presented in Algorithm~\ref{alg:incrpca}, uses two non-convex projection operations as building blocks. Our algorithm essentially applies these projections to the low-rank and sparse residuals in an alternating manner until convergence, i.e., at the $t^{th}$ iteration, the residuals $M-L_{t-1}$ and $M-S_t$ are projected onto the set of sparse and low-rank matrices respectively via the following hard-thresholding operations: 
\begin{enumerate}[nolistsep,noitemsep]
\item \textit{Spectral hard thresholding: }This is used for projecting a matrix onto the set of low-rank matrices. It is achieved via the truncated-SVD operation and is denoted by $B = \P_r(A)$. Here, we are finding a matrix rank-$r$ matrix $B$ which best approximates $A$.
\item \textit{Entry-wise hard thresholding: }This is used for projecting a matrix onto the set of sparse matrices. We compute a matrix $B = \P_a (A)$ where $B_{ij} = A_{ij}$ if $\abs{A_{ij}} > a$ and $B_{ij} = 0$ if $\abs{A_{ij}} \leq a$.
\end{enumerate}
Note that the above hard thresholding operations result in in rank-restricted and sparsity-restricted matrices for appropriate choices of $r$ and $a$. It is noteworthy that our algorithm, unlike many non-convex optimization procedures, employs the very simple initialization scheme of setting the initial iterates to the all-zeros matrix ($L_0$) while achieving global convergence.

The algorithm needs (a) the true rank $r$ of $W^*$, and (b) the noise parameter $\nu$ (for which it suffices to have the knowledge of a reasonable bound on $\infnorm{N^*}(1+3\mu^2 d \kappa^2)$). 
In practice, the knowledge of $r$ and $\infnorm{N^*}$ can be obtained using cross-validation, grid search or leveraging domain knowledge of the specific application; for instance, in the noiseless setting, $N^* = 0$ and hence, $\nu$ is set to zero. Furthermore, efficient ways of estimating the incoherence of a matrix have been studied in the literature; see for instance, \citep{mohri2011can} and \citep{drineas2012fast}.

A key difference from related approaches in the transductive setting~\citep{netrapalli2014non} is the more efficient spectral hard thresholding that is possible due to the available feature information, i.e., our approach involves a truncated SVD operation in the feature space rather than the ambient space which is computationally inexpensive. Specifically, since $L^* = F_1^\top W^* F_2$, in Step 7 of Algorithm~\ref{alg:incrpca}, we find the best matrix $\underline{W}_t$ such that $M-S_t \approx F_1^\top \underline{W}_t F_2$ for every $t$. This is achieved via a bilinear transformation of the residual $M-S_t$ given by ${(F_1^\top)}^\dagger (M-S_t) (F_2)^\dagger$ followed by a truncated $r$-SVD of the resulting $d_1 \times d_2$ matrix $\underline{W}_t$ to obtain $W_t$. Note that the low-rank iterates may then be computed as $L_t = F_1^\top W_t F_2$; specifically, $\widehat{L} = F_1^\top \widehat{W} F_2$ at termination.

\subsection{COMPUTATIONAL COMPLEXITY}
\label{sec:comp_comp}
We now infer the per-iteration computational complexity from Algorithm~\ref{alg:incrpca}, specifically Steps 6-8. The entry-wise hard-thresholding in Step 6 has a time complexity of $O(n_1 n_2)$. The spectral hard-thresholding in Step 7 has a time complexity of $O(\max(n_1^2 d_1,n_2^2 d_2) + d_1 d_2 r)$ due to the involved matrix multiplication followed by the truncated SVD operation. Step 8 has a complexity of $O(n_1 n_2 \max(d_1, d_2))$. Unlike previous~\citep{chiang2016robust} trace norm based approaches in the inductive setting, we directly perform rank-$r$ SVD in Step 7 leading to a complexity of just $O(d_1 d_2 r)$ as opposed to $O(d_1 d_2 \min(d_1, d_2))$; this is a significant gain when $r \ll \min(d_1, d_2)$. In the transductive setting as well, our method has significant computational gains over the state-of-the art AltProj algorithm of \citep{netrapalli2014non}, especially in the regime $\max(d_1,d_2) < r^2$ while maintaining the corruption rate guarantees as in Section~\ref{sec:samp_comp}.


\section{ANALYSIS}
\subsection{PROOF OUTLINE}
For simplicity, we first begin with the symmetric noiseless case (Section~\ref{sec:sym_noiseless}). Upon presenting the convergence result for this case, we show how to extend our analysis and result to general cases including the noisy case (Section~\ref{sec:sym_noisy}) and the asymmetric matrix case (Section~\ref{sec:asymm}).

The key steps in the proof of convergence of Algorithm~\ref{alg:incrpca} involve analyzing the two main hard-thresholding operations and controlling the error decrease, in terms of a suitably chosen potential function, as a result of performing these operations. Since we care about recovering every entry of both the low-rank and the sparse matrix components, we choose the infinity norm of appropriate error matrices as our potential function to track the progress of our algorithm. Bounds in the infinity norm are trickier to obtain than the more usual spectral norm. Consequently, our guarantees are stronger as opposed to showing faithful recovery in the spectral or Frobenius norms. Specifically, for a given $t$, we show that $\infnorm{L^*-L_t} \leq 2 \infnorm{S^*-S_t} \leq \frac{1}{5} \infnorm{L^*-L_{t-1}}$. Upon showing this geometric reduction in error, we use induction to stitch up argument across iterations.

At a high level, the proof techniques involved for a fixed $t$ are as follows:
\begin{enumerate}[noitemsep,nolistsep]
\item \textit{Entry-wise hard thresholding: }The are two aspects here. First, given that $L_{t-1}$ is close to $L^*$, we show, by using a case-by-case argument, that $S_t$ is also close to $S^*$. Second, we show, by contradiction, that the $S_t$ does not have any spurious entries that are not present in $S^*$ originally.
\item \textit{Spectral hard thresholding: }Given that $S_t$ is close to $S^*$, we show that $L_t$ gets closer to $L^*$ than $L_{t-1}$. There are three aspects here. First, we use the weak incoherence property of features to obtain infinity norm bounds. Second, we use Weyl's eigenvalue pertubation lemma to quantify how close the estimate $W_t$ is to the true latent matrix $W^*$. Third, we bound the spectral norm of a sparse matrix tightly in terms of its infinity norm.
\end{enumerate}

For the noisy case, using Remark~\ref{rem:noisy}, we simply account for the noise terms as well in the error reduction argument. Extension to the asymmetric case proceeds via the standard symmetric embedding technique, both for the noiseless and the noisy setting, as detailed in Section~\ref{sec:asymm}; a key point to be noted here is that we maintain the rank-sparsity conditions in the symmetrized matrix.

\subsection{SYMMETRIC NOISELESS CASE}
\label{sec:sym_noiseless}
Let $N^* = 0$, $W^* = {(W^*)}^\top$ and $S^* = {(S^*)}^\top$. For simplicity, let the features be equal i.e., $F_1 = F_2 = F$ and $\mu_{F_1} = \mu_{F_2} = \mu$. Further, let $d_1 = d_2$, $z_1 = z_2 = z$ and $n_1 = n_2 = n$. Also, recall that $\nu = 0$ in the noiseless case. We now state our main result. 

\begin{theorem}[\textbf{Noiseless case: fast and correct convergence}]
\label{thm:sym_noiseless}
Under the assumptions of Section~\ref{sec:assume}, after $T > \lceil \log_5 (2 \mu^2 \sigma_{\max}^2(F) \frac{d}{n} \frac{c_W}{\epsilon}) \rceil + 1$ iterations of Algorithm~\ref{alg:incrpca}, we have $\infnorm{L^*-\widehat{L}} \leq \epsilon$, $\rank(\widehat{L}) \leq r$, $\infnorm{S^*-\widehat{S}} \leq \epsilon$ and $\Supp(\widehat{S}) \subseteq \Supp(S^*)$.
\end{theorem}
\begin{remark}
Several implications are immediate from Theorem~\ref{thm:sym_noiseless}: (1) our algorithm converges to the true parameters at a linear rate; (2) we have faithful latent space recovery as well as outlier detection; (3) assumptions used for deriving the recovery guarantee are weaker than previous works in the inductive setting; (4) we achieve improved corruption rate; (5) guarantees for the transductive robust PCA problem are recovered if the features are identity matrices and $W^* = L^*$; in particular, our corruption rate bounds match up to a factor of $d/r$.
\end{remark}
We now prove Theorem~\ref{thm:sym_noiseless}.
\begin{proof}
We prove this by induction over $t$. Note that Step 3 of Algorithm~\ref{alg:incrpca} initializes $\zeta_0 = 5 \mu^2 \sigma_{\max}^2(F) \frac{d}{n} c_W$ (as $N^* = 0$) and sets $\zeta_t = \zeta_{t-1}/5$ for all $t \geq 1$. For $t = 1$, since $L_{0} = 0$ by our initialization, it is clear that $\infnorm{L^* - L_{0}} \leq \infnorm{L^*} \leq \infnorm{F^\top W^* F} \leq \mu^2 \sigma_{\max}^2(F) \frac{d}{n} c_W$ and hence the base case holds.

Next, for $t \geq 1$, by using Lemma~\ref{lem:sparse}, we have $\infnorm{S^*-S_{t}} \leq 2 \mu^2 \sigma_{\max}^2(F) \frac{d}{n} \frac{c_W}{5^{t-1}}$ and $\Supp(S_{t}) \subseteq \Supp(S^*)$ and further, by Lemma~\ref{lem:decay}, we have $\infnorm{L^*-L_t} \leq \mu^2 \sigma_{\max}^2(F) \frac{d}{n} \frac{c_W}{5^{t}}$. Moreover, setting $T > \lceil \log_5 (2 \mu^2 \sigma_{\max}^2(F) \frac{d}{n} \frac{c_W}{\epsilon}) \rceil + 1$, we have $\infnorm{L^*-L_T} \leq \epsilon$ and $\infnorm{S^*-S_T} \leq \epsilon$.
\end{proof}

\begin{lemma}[\textbf{Noiseless case: faithful support recovery due to entry-wise hard thresholding}]
\label{lem:sparse}
Let $L_{t-1}$ satisfy the error condition that $\infnorm{L^* - L_{t-1}} \leq \mu^2 \sigma_{\max}^2(F) \frac{d}{n} \frac{c_W}{5^{t-1}}$. Then, we have $\infnorm{S^*-S_{t}} \leq 2 \mu^2 \sigma_{\max}^2(F) \frac{d}{n} \frac{c_W}{5^{t-1}}$ and $\Supp(S_{t}) \subseteq \Supp(S^*)$.
\end{lemma}
\begin{proof}
Note that $S_{t} = \P_{\zeta_{t}} (M - L_{t-1}) = \P_{\zeta_{t}} (L^* - L_{t-1} + S^*)$. By the definition of our entry-wise hard thresholding operation, we have the following:
\begin{enumerate}[nolistsep,noitemsep]
\item Term $e_i^\top S_t e_j = e_i^\top (M-L_{t-1}) e_j = e_i^\top (L^*+S^*-L_{t-1}) e_j$ when $\abs{e_i^\top (M - L_{t-1}) e_j} > \zeta_t$.Thus $\abs{e_i^\top (S^*-S_{t}) e^j} = \abs{e_i^\top (L^*-L_{t-1}) e_j} \leq \mu^2 \sigma_{\max}^2(F) \frac{d}{n} \frac{c_W}{5^{t-1}}$.
\item Term $e_i^\top S_t e_j = 0$ when $\abs{e_i^\top (M - L_{t-1}) e_j} = \abs{e_i^\top (L^* + S^* - L_{t-1}) e_j} \leq \zeta_t$. Using the triangle inequality, we have $\abs{e_i^\top (S^*-S_{t}) e_j} \leq \abs{e_i^\top S^* e_j} \leq \zeta_t + \abs{e_i^\top (L^* - L_{t-1}) e_j} \leq 2 \mu^2 \sigma_{\max}^2(F) \frac{d}{n} \frac{c_W}{5^{t-1}}$.
\end{enumerate}
Thus, the above two cases show the validity of the entry-wise hard thresholding operation. To show correct support recovery, we show that for any given $(i,j)$, if $e_i^\top S^* e_j = 0$ then $e_i^\top S_{t} e_j$ is also zero for all $t$. Noting that $M = L^* + S^*$ and $e_i^\top S^* e_j = 0$, $e_i^\top S_{t} e_j = e_i^\top (M - L_{t-1}) e_j = e_i^\top (L^* - L_{t-1}) e_j \neq 0$ iff $\abs{e_i^\top (L^* - L_{t-1}) e_j} > \zeta_t$. But this is a contradiction since $\abs{e_i^\top (L^* - L_{t-1}) e_j} \leq \mu^2 \sigma_{\max}^2(F) \frac{d}{n} \frac{c_W}{5^{t-1}} = \zeta_t$ by the inductive assumption.
\end{proof}

\begin{lemma}[\textbf{Noiseless case: error decay due to spectral hard thresholding}]
\label{lem:decay}
Let $S_{t}$ satisfy the error condition that $\infnorm{S^* - S_{t}} \leq 2 \mu^2 \sigma_{\max}^2(F) \frac{d}{n} \frac{c_W}{5^{t-1}}$. Then, we have $\infnorm{L^*-L_t} \leq \mu^2 \sigma_{\max}^2(F) \frac{d}{n} \frac{c_W}{5^{t}}$ and $\rank(L_t) \leq r$.
\end{lemma}
\begin{proof}
Using the fact that $L^* = F^\top W^* F$ and $L_t = F^\top W_t F$, we have
\begin{align}
& \infnorm{L^*-L_t} = \infnorm{F^\top (W^* - W_t) F} \nn
\end{align}
\begin{align}
& = \max_{i,j} \abs{e_i^\top F^\top (W^* - W_t) F e_j} \nn \\
& \stackrel{\xi_1}{=} \max_{i,j} \abs{e_i^\top V_F \Sigma_F^\top U_F^\top (W^* - W_t) U_F \Sigma_F V_F^\top e_j} \nn \\
& \stackrel{\xi_2}{\leq} \left(\max_{i} \twonorm{e_i^\top V_F \Sigma_F^\top}\right)^2 \twonorm{U_F^\top (W^* - W_t) U_F}, \label{eqn:VW}
\end{align}
where $\xi_1$ follows by substituting the SVD of $F$, i.e., $F = U_F \Sigma_F V_F^\top$ and $\xi_2$ follows from the sub-multiplicative property of the spectral norm. Now, from Assumption \ref{asm:incoh}, we have
\begin{equation}
\label{eqn:incoh_v}
\max_i \twonorm{e_i^\top V_F \Sigma_F^\top} \leq \mu \sqrt{\frac{d}{n}} \sigma_{\max}(F).
\end{equation}
Recall from Step 7 of Algorithm~\ref{alg:incrpca} that $W_t$ is computed as $\P_r \paran{ {(F_1^\top)}^\dagger (M-S_t) (F_2)^\dagger }$ where $M = F_1^\top W^* F_2 + S^*$. Let $E_t := S^*-S_t$. Further, let $Q \Lambda Q^\top + Q_\perp \Lambda_\perp Q_\perp^\top$ be the full SVD of $W^* + G^\top E_t G$, where $Q$ and $Q_\perp$ span orthogonal sub-spaces of dimensions $r$ and $d-r$ respectively, and $G := F^\dagger$ is the pseudoinverse of $F$. 
Next, using these and the unitary invariance property of the spectral norm, we have
\begin{align}
& \twonorm{U_F^\top (W^*-W_t) U_F} \leq \twonorm{W^*-W_t} \nn \\
& \leq \twonorm{W^* - \P_r (G^\top (F^\top W^* F + E_t) G)} \nn \\
& \stackrel{\xi_3}{\leq} \twonorm{Q \Lambda Q^\top + Q_\perp \Lambda_\perp Q_\perp^\top - G^\top E_t G - Q \Lambda Q^\top} \nn \\
& \stackrel{\xi_4}{\leq} \twonorm{G^\top E_t G} + \twonorm{Q_\perp \Lambda_\perp Q_\perp^\top} \nn \\
& \stackrel{\xi_5}{\leq} 2 \twonorm{G^\top E_t G} {\leq} 2 \twonorm{G}^2 \twonorm{E_t} \nn \\
& \leq \frac{2 \twonorm{E_t}}{[\sigma_{\min}(F)]^2} \stackrel{\xi_6}{\leq} \frac{2 z \infnorm{E_t}}{[\sigma_{\min}(F)]^2}, \label{eqn:wwt}
\end{align}
where $\xi_3$ is obtained by substituting $W^* = Q \Lambda Q^\top + Q_\perp \Lambda_\perp Q_\perp^\top - G^\top E_t G$, $\xi_4$ by triangle inequality. Inequality $\xi_5$ is obtained by using Weyl's eigenvalue perturbation lemma~\citep{bhatia2013matrix}, which is:
\begin{align*}
\twonorm{Q_\perp \Lambda_\perp Q_\perp^\top} = \infnorm{\Lambda_\perp} \leq \twonorm{G^\top E_t G}.
\end{align*}
Finally, inequality $\xi_6$ is obtained by using Lemma 4 of \citep{netrapalli2014non}. Combining Equations \eqref{eqn:VW}, \eqref{eqn:incoh_v} and \eqref{eqn:wwt}, we have
\begin{equation}
\label{eqn:llt}
\infnorm{L^*-L_t} \leq 2 \mu^2 d z \kappa^2 \infnorm{E_t} / n \stackrel{\xi_7}{\leq}  \infnorm{E_t} / 10,
\end{equation}
where $\kappa = \frac{\sigma_{\max}(F)}{\sigma_{\min}(F)}$ and $\xi_7$ is due to Assumption~\ref{asm:sps}. Substituting the result $\infnorm{E_t} = \infnorm{S^*-S_{t}} \leq 2 \mu^2 \sigma_{\max}^2(F) \frac{d}{n} \frac{c_W}{5^{t-1}}$ from Lemma~\ref{lem:sparse} in Equation~\eqref{eqn:llt} completes the proof.
\end{proof}

\subsection{SYMMETRIC NOISY CASE}
\label{sec:sym_noisy}

\begin{figure*}[t]
\centering
\caption{Comparing of RPCA algorithms in terms of running time to reach a solution of a given accuracy. For $n = 1000$, we vary each problem parameter while fixing the others. Specifically, we vary:}
\begin{tabular}{cccc}
(a) sparsity & (b) rank & (c) feature dimension & (d) condition number \\
\includegraphics[width = 0.23\textwidth]{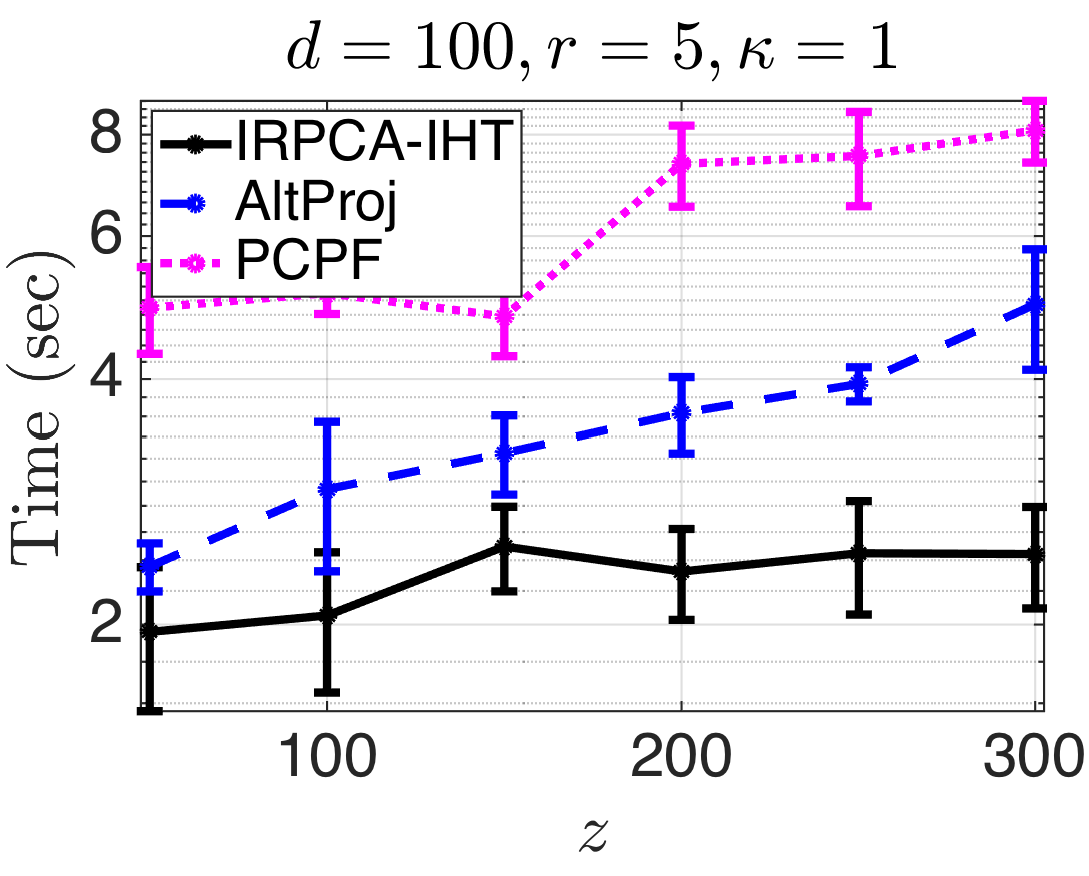} &
\includegraphics[width = 0.23\textwidth]{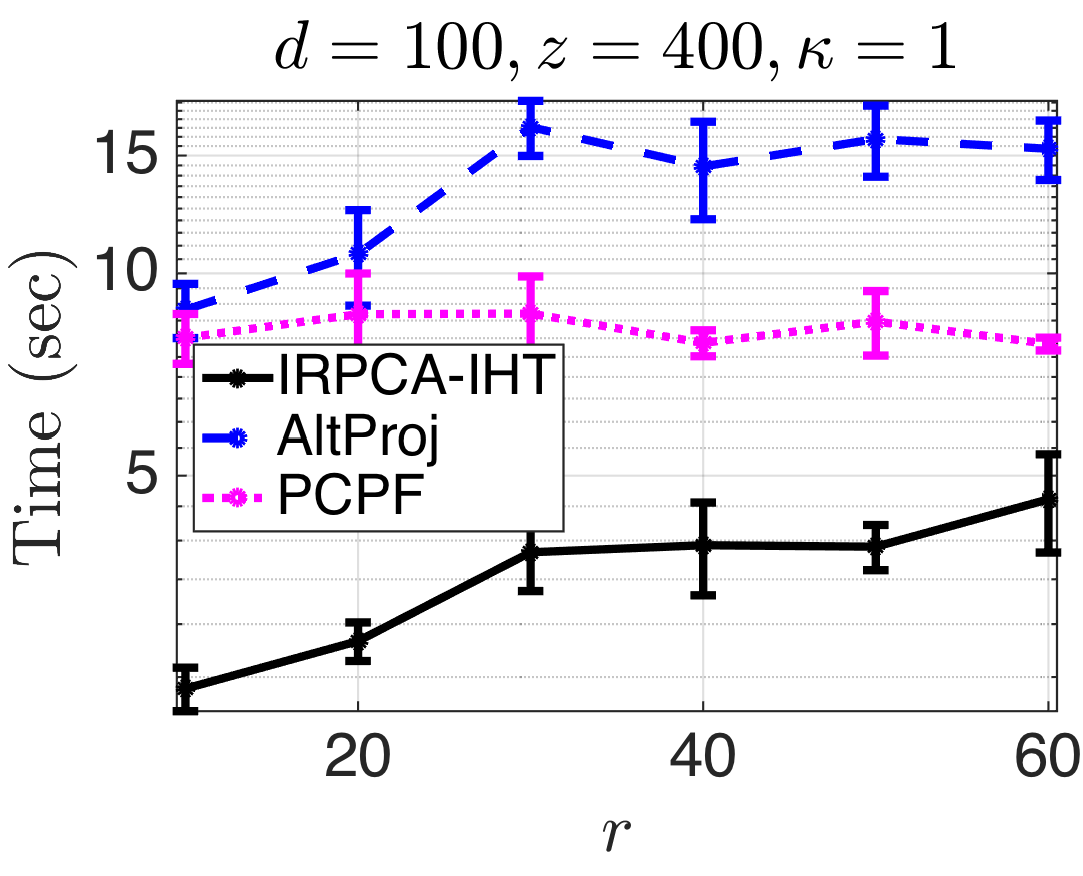} &
\includegraphics[width = 0.23\textwidth]{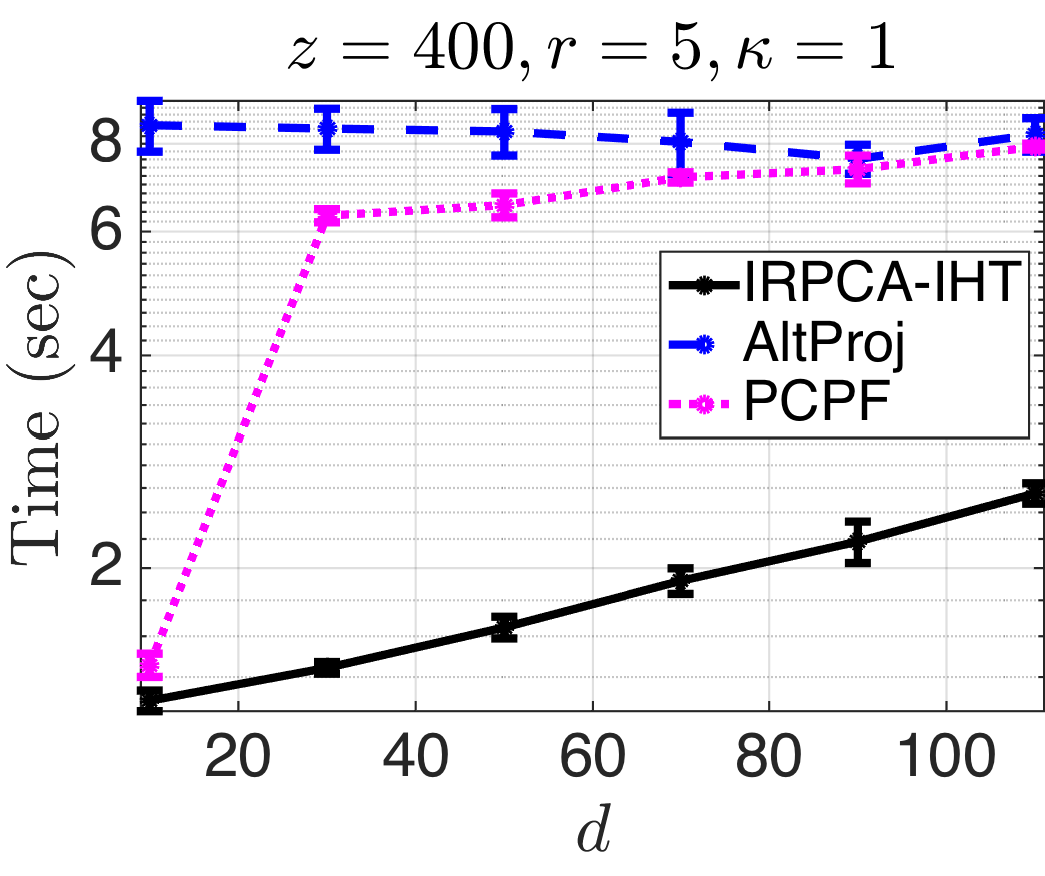} &
\includegraphics[width = 0.23\textwidth]{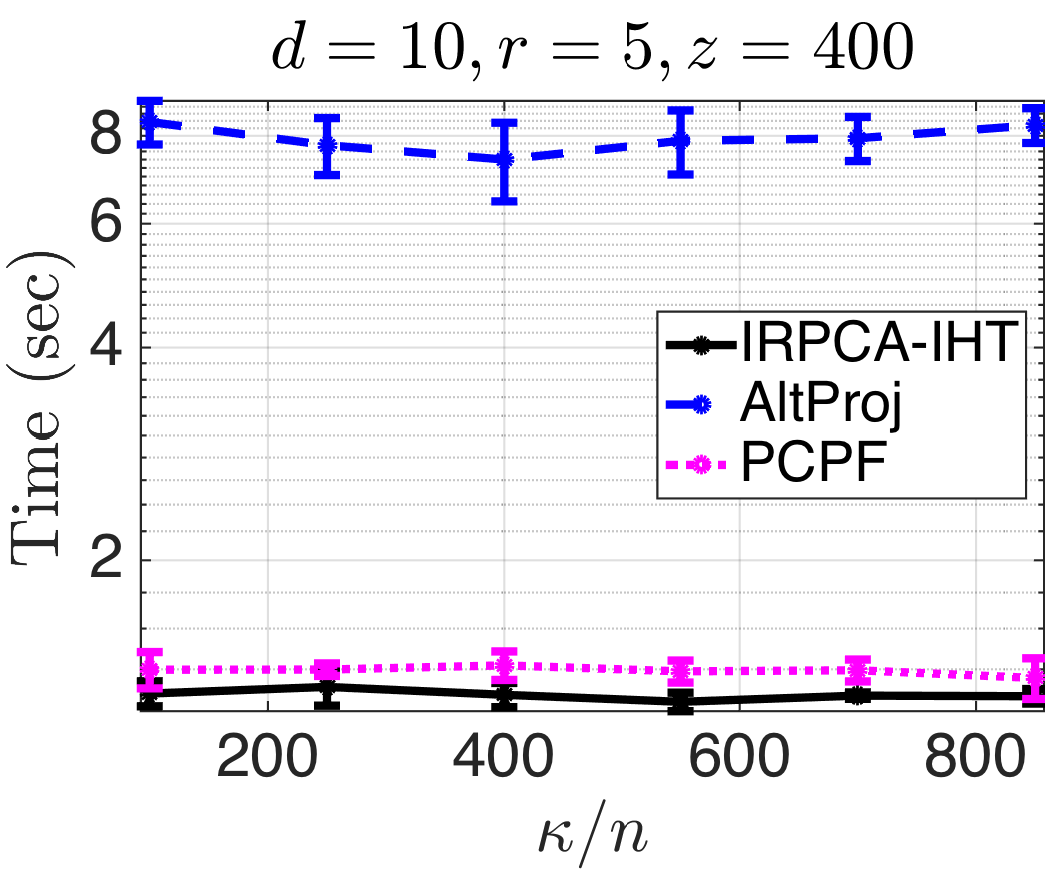}
\end{tabular}
\label{fig:syn}
\vspace*{-20pt}
\end{figure*}

Next we consider the general noisy case of $M = L^* + S^* + N^*$, where $L^* = F^\top W^* F$, $S^*$ and $N^*$ are symmetric and $N^*$ is a bounded additive noise matrix satisfying properties as given in Remark~\ref{rem:noisy_asm}. Note that, in practice, by setting $\nu = c.d$ for a suitably chosen constant $c$, Algorithm~\ref{alg:incrpca} works unchanged. However, in order to establish convergence in theory, the key challenge is to be able to control the perturbation effects of $N^*$ in each iteration. In making this precise, we now state our main result for this section whose proof is given in Appendix~\ref{sec:proofs_noisy} due to space limitations.

\begin{theorem}[\textbf{Noisy case: fast and correct convergence}]
\label{thm:sym_noisy}
Under the assumptions of Section~\ref{sec:assume}, setting $T > \lceil \log_5 (2 \mu^2 \sigma_{\max}^2(F) \frac{d}{n} \frac{c_W}{\epsilon}) \rceil + 1$ in Algorithm~\ref{alg:incrpca}, we have $\infnorm{L^*-\widehat{L}} \leq \epsilon + 3 \mu^2 d \kappa^2 \infnorm{N^*}$, $\rank(\widehat{L}) \leq r$, $\infnorm{S^*-\widehat{S}} \leq \epsilon + 8 \mu^2 d \kappa^2 \infnorm{N^*}$ and $\Supp(\widehat{S}) \subseteq \Supp(S^*)$.
\end{theorem}

To prove the above theorem, we need the following key lemmas whose proofs are given in Appendix~\ref{sec:proofs_noisy} as well.

\begin{lemma}[\textbf{Noisy case: faithful support recovery due to entry-wise hard thresholding}]
\label{lem:sparsen}
Let $L_{t-1}$ satisfy the error condition that $\infnorm{L^* - L_{t-1}} \leq \mu^2 \sigma_{\max}^2(F) \frac{d}{n} \frac{c_W}{5^{t-1}} + 3 \mu^2 d \kappa^2 \infnorm{N^*}$. Then, we have $\infnorm{S^*-S_{t}} \leq 2 \mu^2 \sigma_{\max}^2(F) \frac{d}{n} \frac{c_W}{5^{t-1}} + 2 (3 \mu^2 d \kappa^2 + 1) \infnorm{N^*}$ and $\Supp(S_{t}) \subseteq \Supp(S^*)$.
\end{lemma}

\begin{lemma}[\textbf{Noisy case: error decay due to spectral hard thresholding}]
\label{lem:decayn}
Let $S_{t}$ satisfy the error condition that $\infnorm{S^* - S_{t}} \leq 2 \mu^2 \sigma_{\max}^2(F) \frac{d}{n} \frac{c_W}{5^{t-1}} + 2 (3 \mu^2 d \kappa^2 + 1) \infnorm{N^*}$. Then, we have $\infnorm{L^*-L_t} \leq \mu^2 \sigma_{\max}^2(F) \frac{d}{n} \frac{c_W}{5^{t}} + 3 \mu^2 d \kappa^2 \infnorm{N^*}$ and $\rank(L_t) \leq r$.
\end{lemma}

\subsection{ASYMMETRIC CASE}
\label{sec:asymm}
We now show how to extend our analysis for any general asymmetric matrix, both in the noiseless and the noisy inductive settings. Let $M \in \R^{n_1 \times n_2}$ be the input data matrix. The main result can be stated as:
\begin{claim}
\label{clm:asymm}
Let $M = L^* + S^* + N^*$ where $L^* = F_1^\top W^* F_2$ such that $n_1 \neq n_2$ and $d_1 \neq d_2$. Algorithm~\ref{alg:incrpca} executed on this $M$ satisfies the guarantees in Theorem~\ref{thm:sym_noiseless} (resp. Theorem~\ref{thm:sym_noisy}) for the noiseless case where $N^*=0$ (resp. noisy case where $N^*$ satisfies the properties in Remark~\ref{rem:noisy_asm}).
\end{claim}
Consider the standard symmetric embedding of a matrix given by:
\begin{align*}
\sym(M) := 
\begin{pmatrix}
0 & M \\
M^\top & 0
\end{pmatrix}.
\end{align*}
\vspace*{-2pt}
With $\sym(M)$ as input, the intermediate iterates of our algorithm also have a similar form. Moreover, note that this embedding preserves the rank, incoherence and sparsity properties -- due to space constraints, these details which are needed as the key components of the proof of Claim~\ref{clm:asymm} are deferred to Appendix~\ref{apdx:asymm}.

\section{EXPERIMENTS}
\label{sec:expt}
In this section, we conduct a systematic empirical investigation of the performance of our robust subspace recovery method (IRPCA-IHT) and justify our theoretical claims in the previous sections. Specifically, the goal of this study is to show: (1) the correctness of our algorithm, (2) that informative features and feature correlations are indeed useful, and (3) that our algorithm is computationally efficient.

\subsection{SYNTHETIC SIMULATIONS}
We set the problem size as $n_1 = n_2 = n = 1000$; for simplicity, we take $d_1 = d_2 = d$, $z_1 = z_2 = z$ and $F_1 = F_2 = F$; let $\kappa$ be the condition number of the feature matrix $F$. First, we generate approximately well-conditioned weakly incoherent feature matrices by computing $F = U_{F} \Sigma_{F} V_{F}^\top$ where the entries of $U_{F} \in \R^{d \times d}$ and $V_{F} \in \R^{d \times n}$ are drawn iid from the standard normal distribution followed by row normalization, and the diagonal entries of $\Sigma_{F}$ are set to one. Next, the latent matrix $W^*$ is generated by sampling each entry independently and uniformly at random from the interval $(0,1)$, performing SVD of this sampled matrix and retaining its top $r$ singular values. The low-rank component $L^*$ is then computed as $F^\top W^* F$. Note that this also ensure the feasibility condition in Assumption~\ref{asm:feas}. Next, we generate the sparse matrix as follows. We first choose the support according to the Bernoulli sampling model, ie, each entry is chosen to be included in the support with probability $z/n$ and then its value is chosen independently and uniformly at random from $(-10r/n,-5r/n) \cup (5r/n,10r/n)$.

There are four main parameters in the problem namely, (a) the sparsity level $z$ of $S^*$, (b) the rank $r$ of $W^*$, (c) the feature dimension $d$, and (d) condition number $\kappa$ of the feature matrix $F$; we vary each of these while fixing the others. We compare the performance of our algorithm to that of two existing algorithms namely, (i) the convex relaxation approach `PCPF' due to \citep{chiang2016robust} which is a state-of-the-art robust PCA method in the inductive setting, and (ii) `AltProj' due to \citep{netrapalli2014non} which is a state-of-the-art robust PCA method in the transductive setting. We execute these algorithms until an accuracy of $\frobnorm{M-\widehat{L}-\widehat{S}} / \frobnorm{M} \leq 10^{-3}$ is achieved and time them individually. All the results presented in the running time plots in Figure~\ref{fig:syn} are obtained by averaging over five runs.

We note that our algorithm outperforms PCPF and AltProj consistently while increasing the problem hardness in three situations (Figures 1-(a), 1-(b) and 1-(c)) in terms of running time. The gain in terms of scalability of our method over the convex PCPF method is attributed to the fact that the soft thresholding operation for solving the nuclear-norm objective involves computing the partial-SVD of the intermediate iterates which could of potentially much higher rank than $r$ -- this leads to $O(d^3)$ worst-case time complexity for the SVD step in PCPF as opposed to our algorithm which has $O(d^2 r)$ worst-case complexity for spectral hard thresholding. The time gain over the transductive AltProj method is attributed to the fact that our spectral hard-thresholding is performed in the $d$-dimensional (feature) space rather than the $n$-dimensional (ambient) space; moreover, another factor that adds to the running time of AltProj is that it proceeds in stages unlike Algorithm~\ref{alg:incrpca}. An interesting point to be noted from the relatively flat plot in Figure~\ref{fig:syn}-(d) is that the condition number dependence in Assumption~\ref{asm:sps} is merely an artifact of our analysis and is not inherent to the problem; we leave tightening this bound in theory to future work.

\begin{figure}[t]
\caption{Comparison of robust PCA algorithms on the MovieLens data: running time and recovery error.}
\label{fig:mlens}
\includegraphics[width=0.45\columnwidth]{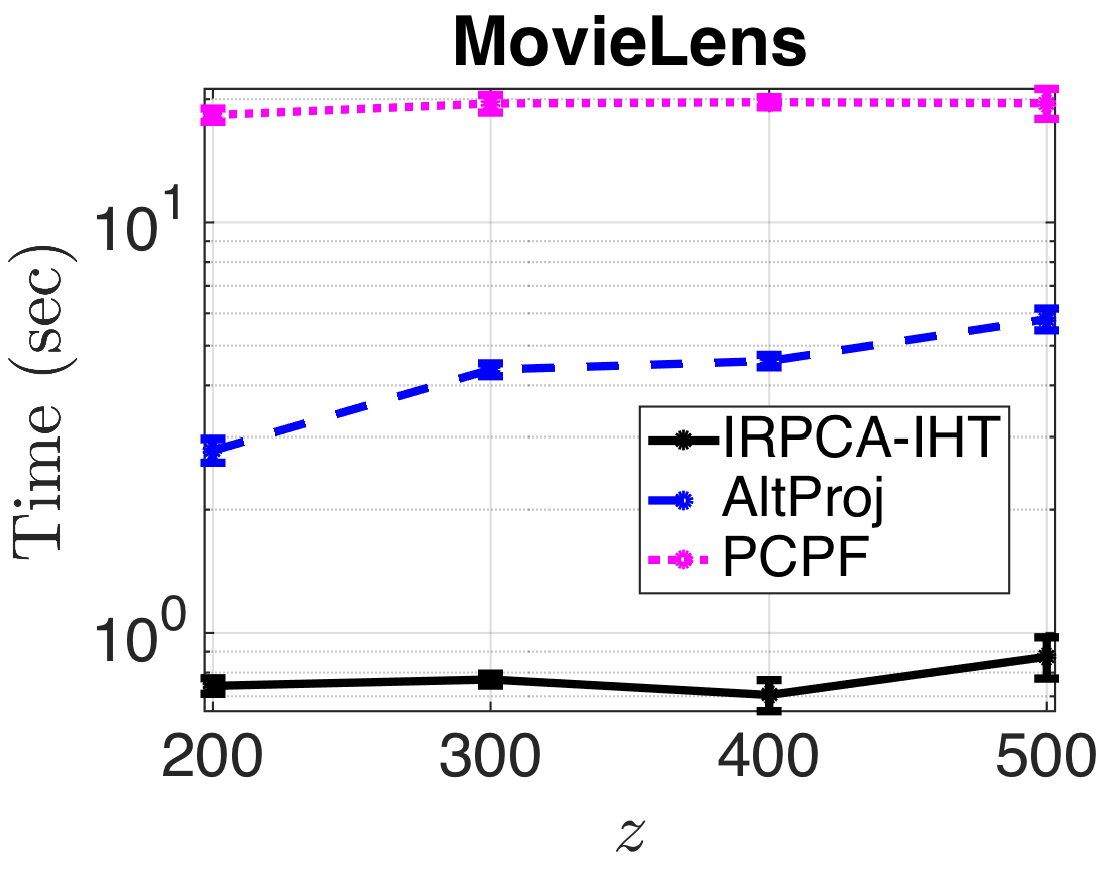}
\includegraphics[width=0.45\columnwidth]{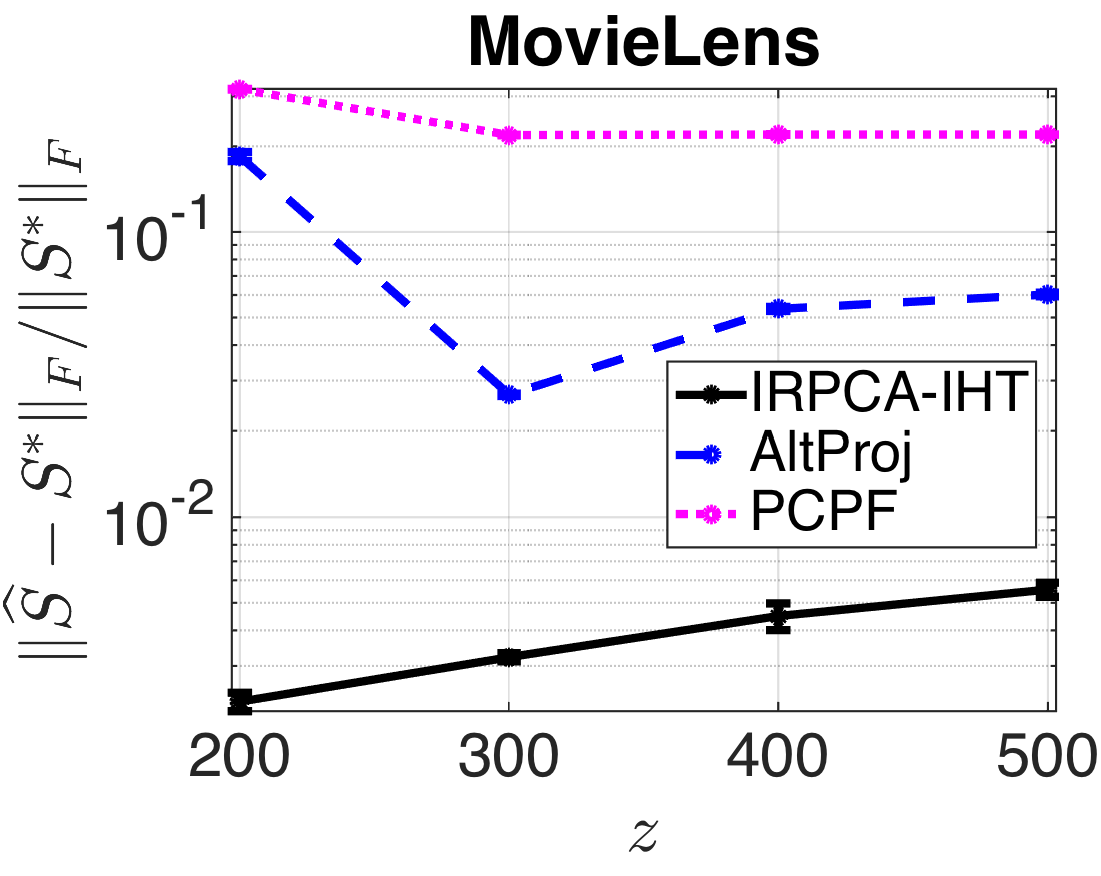}
\end{figure}

\subsection{REAL-DATA EXPERIMENTS}
As described in Example~\ref{eg:reco_sys}, we consider an important application of our method -- to robustify estimation in recommendation systems while leveraging feature information; specifically, the task is to predict user-movie ratings accurately despite the presence of gross sparse corruptions. We take the MovieLens~\footnote{\url{http://grouplens.org/datasets/movielens/}} dataset which consists of $100,000$ ratings from $n_1 = 943$ users on $n_2 = 1682$ movies. The ground-truth in this dataset is, per se, unavailable. Hence, as the first step, we apply matrix completion techniques (specifically, using the OptSpace algorithm of \citep{keshavan2010matrix}) to obtain a baseline complete user-movie ratings matrix, $L^*$; we take $r=3$. Next, we form features while ensuring the feasibility condition. For this, we compute the SVD of the baseline matrix, $L^* = U_{L^*} \Sigma_{L^*} V_{L^*}^\top$ followed by setting $F_1 = U_{L^*} Q_U$ (resp. $F_2 = V_{L^*} Q_V$) where $Q_U \in SO(d_1)$ (resp. $Q_V \in SO(d_2)$) are random rotation matrices; we take $d_1 = 20$ and $d_2 = 25$. Note that forming features using the SVD result, as we have done here, is a common technique in inductive matrix estimation problems (see, for instance, \citep{natarajan2014inductive}). We then add a sparse perturbation matrix whose each entry is chosen to be included in the support with probability $z/n$ and the entries are chosen independently and uniformly at random from $(-10r/\sqrt{n_1 n_2},-5r/\sqrt{n_1 n_2}) \cup (5r/\sqrt{n_1 n_2},10r/\sqrt{n_1 n_2})$. We compare the performance of PCPF, AltProj and our IRPCA-IHT algorithms; we consider two evaluation metrics, running time and relative recovery error (the latter is measured by $\frobnorm{\widehat{S}-S^*}/\frobnorm{S^*}$). Varying $z$ and averaging over five runs, we note that our algorithm outperforms (Figure \ref{fig:mlens}) both PCPF and AltProj by achieving about an order of magnitude of gain in terms of both the running time as well as the recovery error.

\section{CONCLUSION}
In this paper, we have presented an novel approach for inductive robust subspace identification by leveraging available informative feature information. We hope our results motivate similar studies of other learning problems in the inductive setting leading to improved statistical and computational performance. Keeping this in mind, some future directions with respect to this work include understanding the following:
\begin{enumerate}[nolistsep,noitemsep]
\item Minimax rates, both tight lower and upper bounds for learning problems in the inductive setting, are of interest. Relevant techniques include the works by \citep{negahban2012restricted} and \citep{klopp2014robust} in the transductive setting.
\item We note that the corruption rate in Assumption~\ref{asm:sps} is still sub-optimal by a factor of $d/r$ which is significant when $r \ll d$. In addition to this, removing the condition number dependence in Assumption~\ref{asm:sps} and also obtaining $\epsilon$-independent results as in matrix completion (see for instance, \citep{jain2014fast}) are of interest.
\end{enumerate}

\newpage
\appendix

\onecolumn
\begin{center}
\textbf{\Large{Appendix of \\ Provable Inductive Robust PCA via Iterative Hard Thresholding}}
\end{center}

\section{PROOFS: NOISY CASE}
\label{sec:proofs_noisy}

\subsection{Proof of Theorem~\ref{thm:sym_noisy}}
\begin{proof}
We prove this by induction over $t$. Note that Step 3 of Algorithm~\ref{alg:incrpca} initializes $\zeta_0 = 5 \mu^2 \sigma_{\max}^2(F) \frac{d}{n} c_W + \nu$ and sets $\zeta_t = \mu^2 \sigma_{\max}^2(F) \frac{d}{n} \frac{c_W}{5^{t-1}} + \nu$ for all $t \geq 1$. Let $\nu = (3 \mu^2 d \kappa^2 + 1) \infnorm{N^*}$. For $t = 1$, since $L_{0} = 0$ by our initialization, it is clear that $\infnorm{L^* - L_{0}} \leq \mu^2 \sigma_{\max}^2(F) \frac{d}{n} c_W$ and hence the base case holds. Next, for $t \geq 1$, by using Lemma~\ref{lem:sparsen}, we have $\infnorm{S^*-S_{t}} \leq 2 \mu^2 \sigma_{\max}^2(F) \frac{d}{n} \frac{c_W}{5^{t-1}}  + 2 (3 \mu^2 d \kappa^2 + 1) \infnorm{N^*}$ and further, by Lemma~\ref{lem:decayn}, we have $\infnorm{L^*-L_t} \leq \mu^2 \sigma_{\max}^2(F) \frac{d}{n} \frac{c_W}{5^{t}} + 3 \mu^2 d \kappa^2 \infnorm{N^*}$. Moreover, setting $T > \lceil \log_5 (2 \mu^2 \sigma_{\max}^2(F) \frac{d}{n} \frac{c_W}{\epsilon}) \rceil + 1$, we obtain the result.
\end{proof}

\subsection{Proof of Lemma~\ref{lem:sparsen}}
\begin{proof}
Recall that $S_{t} = \P_{\zeta_{t}} (M - L_{t-1}) = \P_{\zeta_{t}} (L^* - L_{t-1} + S^* + N^*)$. By the definition of our entry-wise hard thresholding operation, we have the following:
\begin{enumerate}[nolistsep,noitemsep]
\item Term $e_i^\top S_t e_j = e_i^\top (M-L_{t-1}) e_j = e_i^\top (L^*+S^*+N^*-L_{t-1}) e_j$ when $\abs{e_i^\top (M - L_{t-1}) e_j} > \zeta_t$. Thus, $\abs{e_i^\top (S^*-S_{t}) e_j} = \abs{e_i^\top (L^*-L_{t-1}) e_j} + \abs{e_i^\top N^* e_j} \leq \mu^2 \sigma_{\max}^2(F) \frac{d}{n} \frac{c_W}{5^{t-1}} + 3 \mu^2 d \kappa^2 \infnorm{N^*} + \infnorm{N^*}$.
\item Term $e_i^\top S_t e_j = 0$ when $\abs{e_i^\top (M - L_{t-1}) e_j} = \abs{e_i^\top (L^* + S^* + N^* - L_{t-1}) e_j} \leq \zeta_t$. Now, using the triangle inequality, we have $\abs{e_i^\top (S^*-S_{t}) e^j} = \abs{e_i^\top S^* e^j} \leq \zeta_t + \abs{e_i^\top (L^* - L_{t-1}) e_j} + \abs{e_i^\top N^* e_j} \leq 2 \paran{ \mu^2 \sigma_{\max}^2(F) \frac{d}{n} \frac{c_W}{5^{t-1}} + 3 \mu^2 d \kappa^2 \infnorm{N^*} + \infnorm{N^*} }$.
\end{enumerate}
Thus, the above two cases show the validity of the entry-wise hard thresholding operation. Next, we show that for any given $(i,j)$, if $e_i^\top S^* e_j = 0$ then $e_i^\top S_{t} e_j$ is also zero for all $t$. Noting that $M = L^* + S^* + N^*$ and $e_i^\top S^* e_j = 0$, $e_i^\top S_{t} e_j = e_i^\top (M - L_{t-1}) e_j = e_i^\top (L^* + N^* - L_{t-1}) e_j \neq 0$ iff $\abs{e_i^\top (L^* + N^* - L_{t-1}) e_j} > \zeta_t$. But this is a contradiction since $\abs{e_i^\top (L^* + N^* - L_{t-1}) e_j} \leq \abs{e_i^\top (L^* - L_{t-1}) e_j} + \abs{e_i^\top N^*e_j} \leq \mu^2 \sigma_{\max}^2(F) \frac{d}{n} \frac{c_W}{5^{t-1}} + 3 \mu^2 d \kappa^2 \infnorm{N^*} + \infnorm{N^*} = \zeta_t$.
\end{proof}

\subsection{Proof of Lemma~\ref{lem:decayn}}
\begin{proof}
Using the fact that $F_1 = F_2$, $L^* = F^\top W^* F$ and $L_t = F^\top W_t F$, we have
\begin{align}
& \infnorm{L^*-L_t}  = \infnorm{F^\top (W^* - W_t) F} \nn \\
& = \max_{i,j} \abs{e_i^\top F^\top (W^* - W_t) F e_j} \nn \\
& \stackrel{\xi_{11}}{=} \max_{i,j} \abs{e_i^\top V_F \Sigma_F^\top U_F^\top (W^* - W_t) U_F \Sigma_F V_F^\top e_j} \nn \\
& \stackrel{\xi_{12}}{\leq} \left(\max_{i} \twonorm{e_i^\top V_F \Sigma_F^\top}\right)^2 \twonorm{U_F^\top (W^* - W_t) U_F} \label{eqn:VWn}
\end{align}
where $\xi_{11}$ follows by substituting the SVD of $F = U_F \Sigma_F V_F^\top$ and $\xi_{12}$ follows from the sub-multiplicative property of the spectral norm. Similar to the proof of Lemma \ref{lem:decay}, using Assumption \ref{asm:incoh} we have:
\begin{equation}
\label{eqn:incoh_vn}
\max_i \twonorm{e_i^\top V_F \Sigma_F^\top} \leq \mu \sqrt{\frac{d}{n}} \sigma_{\max}(F).
\end{equation}
Let the residual sparse perturbation be defined as $E_t := S-S_t$. Let $Q \Lambda Q^\top + Q_\perp \Lambda_\perp Q_\perp^\top$ be the full SVD of $W^* + G^\top (E_t + N^*) G$ where $Q$ and $Q_\perp$ span orthogonal sub-spaces of dimensions $r$ and $d-r$ respectively, and $G = F^\dagger$ is the pseudoinverse. Also, recall that from Step 7 of Algorithm~\ref{alg:incrpca} that $W_t$ is computed as $\P_r \paran{ {(F_1^\top)}^\dagger (M-S_t) (F_2)^\dagger }$ where $M = F_1^\top W^* F_2 + S^* + N^*$. Using these and the unitary invariance property of the spectral norm, we have
\begin{align}
& \twonorm{U_F^\top (W^*-W_t) U_F} \leq \twonorm{W^*-W_t} \nn \\
& \leq \twonorm{W^* - \P_r (G^\top (F^\top W^* F + E_t + N^*) G)} \nn \\
& \stackrel{\xi_{13}}{\leq} \twonorm{Q \Lambda Q^\top + Q_\perp \Lambda_\perp Q_\perp^\top - G^\top (E_t+N^*) G - Q \Lambda Q^\top} \nn \\
& \stackrel{\xi_{14}}{\leq} \twonorm{G^\top (E_t+N^*) G} + \twonorm{Q_\perp \Lambda_\perp Q_\perp^\top} \nn \\
& \stackrel{\xi_{15}}{\leq} 2 \twonorm{G^\top (E_t+N^*) G} {\leq} 2 \twonorm{G}^2 \twonorm{E_t+N^*} \nn \\
& \leq \frac{2 \twonorm{E_t+N^*}}{[\sigma_{\min}(F)]^2} \stackrel{\xi_{16}}{\leq} \frac{2 z \infnorm{E_t}}{[\sigma_{\min}(F)]^2} + \frac{2 \twonorm{N^*}}{[\sigma_{\min}(F)]^2} \label{eqn:wwtn}
\end{align}
where $\xi_{13}$ is obtained by substituting $W^* = Q \Lambda Q^\top + Q_\perp \Lambda_\perp Q_\perp^\top - G^\top (E_t+N^*) G$, $\xi_{14}$ by triangle inequality, $\xi_{15}$ by using Weyl's eigenvalue perturbation lemma, ie,
\begin{align*}
\twonorm{Q_\perp \Lambda_\perp Q_\perp^\top} = \infnorm{\Lambda_\perp} \leq \twonorm{G^\top (E_t+N^*) G}
\end{align*}
and $\xi_{16}$ by using Lemma 4 of \citep{netrapalli2014non} along with triangle inequality. Now, combining Equations \eqref{eqn:VWn}, \eqref{eqn:incoh_vn} and \eqref{eqn:wwtn}, we have
\begin{align}
\infnorm{L^*-L_t} & \leq 2 \mu^2 \frac{d}{n} \kappa^2 \paran{ z \infnorm{E_t} + \twonorm{N^*} } \nn \\
& \stackrel{\xi_{17}}{\leq}  \frac{\infnorm{E_t}}{10} + 2 \mu^2 d \kappa^2 \infnorm{N^*} \label{eqn:lltn}
\end{align}
where $\xi_{17}$ follows by using Assumption~\ref{asm:sps} and the inequality that $\twonorm{N^*} \leq n \infnorm{N^*}$. Using the inequality $\infnorm{S^*-S_{t}} \leq 2 \paran{ \mu^2 \sigma_{\max}^2(F) \frac{d}{n} \frac{c_W}{5^{t-1}} + (3 \mu^2 d \kappa^2 + 1) \infnorm{N^*} }$ from Lemma~\ref{lem:sparsen} in Equation~\eqref{eqn:lltn} completes the proof.
\end{proof}

\section{PROOFS: ASYMMETRIC CASE}
\label{apdx:asymm}
\subsection{Proof of Claim~\ref{clm:asymm}}

\begin{proof}
Applying the symmetric embedding transformation to our data matrix, we get $\sym(M) = \sym(L^*)+\sym(S^*)$. Now we characterize the properties of this symmetric embedding and show that it satisfies Assumptions \ref{asm:feas}, \ref{asm:incoh} and \ref{asm:sps}. First, we have
\begin{align*}
\sym(L^*) & = \begin{pmatrix} 0 & L^* \\ {L^*}^\top & 0 \end{pmatrix} = \begin{pmatrix} 0 & F_1^\top W^* F_2 \\ F_2^\top {W^*}^\top F_1 & 0 \end{pmatrix} \\
& = \begin{pmatrix} F_1^\top & 0 \\ 0 & F_2^\top \end{pmatrix}
\begin{pmatrix} 0 & W^* \\ {W^*}^\top & 0
\end{pmatrix} \begin{pmatrix} F_1 & 0 \\ 0 & F_2 \end{pmatrix}.
\end{align*}
Thus, $\sym(L^*)$ is of the form $\widetilde{F}^\top \widetilde{W}^* \widetilde{F}$. If the SVD of $W^*$ is $U_{W^*} \Sigma_{W^*} V_{W^*}^\top$, then the eigenvalue decomposition of $\widetilde{W}^*$ is given by
\begin{align*}
& \widetilde{W}^* = \begin{pmatrix} 0 & W^* \\ {W^*}^\top & 0 \end{pmatrix} = \begin{pmatrix} 0 & U_{W^*} \Sigma_{W^*} V_{W^*}^\top \\ V_{W^*} \Sigma_{W^*}^\top U_{W^*}^\top & 0 \end{pmatrix} \\
& = \frac{1}{2}\begin{pmatrix} U_{W^*} & U_{W^*} \\ V_{W^*} & -V_{W^*} \end{pmatrix} \begin{pmatrix} \Sigma_{W^*} & 0 \\ 0 & -\Sigma_{W^*} \end{pmatrix} \begin{pmatrix} U_{W^*} & U_{W^*} \\ V_{W^*} & -V_{W^*} \end{pmatrix}^\top,
\end{align*}
implying that $\rank(\widetilde{W}^*) = 2 \cdot \rank(W^*)$. Next, let the SVDs of $F_1$ and $F_2$ be $U_{F_1} \Sigma_{F_1} V_{F_1}^\top$ and $U_{F_2} \Sigma_{F_2} V_{F_2}^\top$ respectively; also, without loss of generality, let $\sigma_{\min}(F_1) > \sigma_{\min}(F_2)$. Then, the SVD of $\widetilde{F} = U_{\widetilde{F}} \Sigma_{\widetilde{F}} V_{\widetilde{F}}^\top $ is given by
\begin{align*}
\widetilde{F} & = \begin{pmatrix} F_1 & 0 \\ 0 & F_2 \end{pmatrix} = \begin{pmatrix} U_{F_1} \Sigma_{F_1} V_{F_1}^\top & 0 \\ 0 & U_{F_2} \Sigma_{F_2} V_{F_2}^\top \end{pmatrix} \\
& = \begin{pmatrix} U_{F_1} & 0 \\ 0 & U_{F_2} \end{pmatrix} \begin{pmatrix} \Sigma_{F_1} & 0 \\ 0 & \Sigma_{F_2} \end{pmatrix} \begin{pmatrix} V_{F_1}^\top & 0 \\ 0 & V_{F_2}^\top \end{pmatrix}
\end{align*}
Now, we verify that the right singular vectors of this new feature matrix $\widetilde{F}$ satisfies weak incoherence property. Specifically, we expect that the following holds:
\begin{equation}
\label{eqn:inc1}
\max_j \twonorm{V_{\widetilde{F}} e_j} \leq \mu_{\widetilde{F}} \sqrt{\frac{d_1+d_2}{n_1+n_2}}
\end{equation}
On the other hand, we actually have
\begin{equation}
\label{eqn:inc2}
\max_j \twonorm{V e_j} \leq \max \paran{\mu_{F_1} \sqrt{\frac{d_1}{n_1}}, \mu_{F_2} \sqrt{\frac{d_2}{n_2}}}.
\end{equation}
Wlog, let $\mu_{F_1} \sqrt{d_1/n_1} > \mu_{F_2} \sqrt{d_2/n_2}$. Then, combining Equations \eqref{eqn:inc1} and \eqref{eqn:inc2}, we want $\frac{\mu_{\widetilde{F}}}{\mu_{F_1}} \leq \sqrt{\frac{1+n_2/n_1}{1+d_2/d_1}}$. In particular, when $n_2/n_1 = d_2/d_1$, the incoherence constant for $\widetilde{F}$ satisfies $\mu_{\widetilde{F}} = \mu_{F_1}$.

Next, note that $\sym(S^*)$ is also sparse; specifically, $\zinorm{S^*} \leq z$ and $\iznorm{S^*} \leq z$ where $z = \max(z_1, z_2)$.

Finally, our algorithm and guarantees hold for general matrices with noise, similar to noiseless case, due to the following observation: $\infnorm{\sym(N^*)} = \infnorm{N^*}$.
\end{proof}

\end{document}